\documentclass[11pt]{article}

\usepackage[preprint]{acl}
\usepackage{times}
\usepackage{latexsym}
\usepackage[T1]{fontenc}
\usepackage[utf8]{inputenc}
\usepackage{microtype}
\usepackage{inconsolata}
\usepackage{amsmath,amssymb,amsthm}
\usepackage{graphicx}
\usepackage{subcaption}
\usepackage{booktabs}
\usepackage{array}

\newtheorem{lemma}{Lemma}
\newtheorem{theorem}{Theorem}
\newtheorem{proposition}{Proposition}
\newtheorem{definition}{Definition}

\newcommand{\F}{\mathcal{F}}
\newcommand{\E}{\mathbb{E}}
\newcommand{\clip}{\mathrm{clip}}
\newcommand{\eps}{\varepsilon}
\newcommand{\Eu}{E^{+}}
\newcommand{\El}{E^{-}}
\newcommand{\Hzu}{H_{0}^{+}}
\newcommand{\alphap}{\alpha_{+}}
\newcommand{\alpham}{\alpha_{-}}

\title{Calibration-Free Early Stopping \\for Chain-of-Thought Reasoning via E-processes}

\author{
 \textbf{Dong-Jae Lee}\thanks{Equal contribution.} \quad
 \textbf{Gyojin Han}$^*$ \quad
 \textbf{Junmo Kim} \quad
\\
 \textsuperscript{}KAIST
\\
    \texttt{\{jhtwosun,
    hangj0820,
    junmo.kim\}@kaist.ac.kr}
}

\begin{document}
\maketitle

\begin{abstract}
Reasoning large language models (LLMs) benefit from test-time scaling, motivating adaptive policies that spend tokens when they help and stop early otherwise. Existing rules halt when a confidence or stability signal crosses a threshold, but such thresholds rarely transfer across models and datasets, and validation-calibrated rules lose their guarantee under shifts in test-time horizon or signal dynamics. We reframe adaptive stopping as sequential evidence accumulation along the reasoning trajectory rather than threshold calibration on held-out data. Our framework runs two e-processes on a shared per-chunk signal: a commit stream tracks the signal's position relative to its slow trend, and an abstain stream tracks its signed velocity, making the two anti-correlated by construction. Under signal-level nulls on the per-chunk increments, Ville's inequality yields an anytime-valid signal-level false-alarm bound at user-specified significance levels, with no held-out calibration. Across various model--dataset combinations, the framework matches or exceeds calibration-based stopping on the token--accuracy Pareto front using zero calibration data, and the same significance level transfers unchanged across cells where calibration becomes infeasible.
\end{abstract}

\section{Introduction}

Reasoning large language models (LLMs)~\cite{yang2025qwen3, guo2025deepseek} benefit from test-time scaling: allocating more reasoning tokens often improves performance on mathematical, scientific, and symbolic tasks~\cite{snell2025scaling, brown2024large}. Long chain-of-thought generation has accordingly become a standard inference-time tool, but it also creates a compute-allocation problem. Some instances are solved early and continue to consume tokens after the answer has effectively stabilized; others exhaust the full budget without making useful progress. The central challenge is to make models reason adaptively: spending computation when it improves reliability, and stopping when further reasoning is unlikely to change the answer.

Recent adaptive-reasoning methods \cite{yang2025deer, fu2025deepconf, aggarwal2023adaptive, li2024esc} instantiate this idea by monitoring a confidence or stability signal during generation and halting once the signal crosses a threshold.
This converts a fixed-budget policy into an instance-dependent stopping rule, but only by trading budget selection for threshold selection: thresholds on confidence, entropy, or stability rarely transfer across models and datasets.
Risk-control approaches such as Conformal Thinking (CT) \cite{wang2026conformal} address this by calibrating thresholds to a user-specified error rate on a held-out validation set.
The guarantee, however, is conditional on the calibration distribution matching deployment: adaptive reasoning runs across many combinations of models, benchmarks, and prompt formats, where per-cell calibration is often unavailable or infeasible at strict risk targets under modest validation budgets.
A source-calibrated rule continues to run when transferred to a new cell, but its guarantee no longer applies.

The difficulty is structural: a chain-of-thought trace is not a single prediction but a sequence of dependent intermediate states, each offering another opportunity to stop.
The event to control is therefore whether an unsafe stopping rule ever fires along the trajectory, not whether a single held-out example crosses a threshold.
Calibrating before deployment ties the rule to one cell's statistics.
When the test-time horizon, signal dynamics, or instance mix shifts, the calibrated mapping no longer matches the trajectory at hand.
A stopping rule must therefore handle the sequential structure of the trajectory without depending on calibration data that may be unavailable or fail to transfer.

We address both requirements by introducing e-processes \cite{ramdas2023game, 10.1111/rssa.12647} to the chain-of-thought stopping problem: sequential evidence accumulators whose error guarantees hold at every stopping time, with no held-out calibration.
We instantiate them as CES (Calibration-Free Early Stopping), a sequential-evidence framework for adaptive chain-of-thought stopping.
Rather than testing a raw signal against a calibrated threshold, CES accumulates evidence along the reasoning trajectory and stops the first time this evidence is sufficient to either commit to or rule out the current answer.
The strength of evidence is governed by a user-controlled significance level $\alpha$: under a signal-level null on the per-chunk increment, Ville's inequality \cite{ville1939etudecritique} bounds the false-alarm probability over the entire trajectory by $\alpha$ at any stopping time.
Because this bound is signal-level, the rule requires no held-out calibration, and $\alpha$ retains its meaning across models, datasets, and token budgets.

We further propose a dual construction with two e-processes operating on a shared signal $s_t$, which can be any of the trajectory statistics already explored in adaptive-reasoning work (e.g., confidence, entropy, answer stability).
The two streams target complementary decisions: the commit stream tracks the position of $s_t$ relative to its own slow trend, and the abstain stream tracks its signed velocity.
They are anti-correlated by construction: an ascending $s_t$ feeds commit evidence and depletes abstain evidence in the same step, so they rarely fire together.
Together, the streams form a symmetric exit rule on the shared signal, so the stopping decision does not depend on any single mapping from signal to correctness.

Our contributions are: (i) We introduce CES, a sequential-evidence framework for adaptive chain-of-thought stopping that provides an anytime-valid signal-level guarantee via Ville's inequality, with a single user-controlled confidence parameter $\alpha$ and no held-out calibration data.
(ii) We propose a dual e-process construction on a shared signal that handles commit and abstain symmetrically under a single decision rule, using a position-velocity decomposition that makes the two streams anti-correlated by construction.
(iii) We show that CES matches or exceeds Conformal Thinking's token--accuracy Pareto across various model--dataset combinations using zero calibration data, and applies unchanged under cross-cell transfer where CT's calibration becomes infeasible at strict risk targets or fails to transfer.

\section{Related Work}

\paragraph{Adaptive stopping for LLM reasoning.}
Recent work modulates test-time compute by monitoring an intermediate signal: token-level entropy or top-1 probability \citep{zhang2025entroduction, yang2025deer}, answer-token entropy at intermediate predictions \citep{wang2025eat}, intermediate-answer self-consistency \citep{wang2023selfconsistency, aggarwal2023adaptive, li2024esc}, learned probes \citep{ni2025uheads}, or budget-aware policies \citep{snell2025scaling,fu2025deepconf}.
These methods either pre-set a threshold heuristically or calibrate it on a held-out set.

\paragraph{Conformal prediction for LLMs.}
Conformal prediction methods derive split-conformal guarantees from a held-out calibration set \citep{angelopoulos2023gentle,abbasi-yadkori2024-conformal-abstention}.
Conformal Thinking \citep{wang2026conformal} adapts split conformal to the reasoning-trace setting with a dual upper/lower stopping rule controlled by a UCB-Hoeffding margin.
The deployed bound is a fixed-time empirical guarantee tied to the calibration distribution; the bound does not transfer to a domain with shifted trace length or signal dynamics.

\paragraph{E-processes and anytime-valid inference.}
E-processes~\cite{ramdas2023game, 10.1111/rssa.12647} are sequential evidence accumulators built from a betting metaphor: a fictitious gambler stakes a fraction of current wealth against the null hypothesis at each step. Under $H_0$, the gambler cannot profit in expectation, so large accumulated wealth constitutes evidence against $H_0$. Ville's inequality~\citep{ville1939etudecritique} formalizes this: for any nonnegative supermartingale, the probability of the wealth \emph{ever} exceeding $1/\alpha$ is at most $\alpha$, at any stopping time~\citep{ramdas2023game,10.1111/rssa.12647}. This time-uniformity distinguishes e-processes from classical $p$-values, which require a fixed sample size and break under adaptive stopping.
The betting fraction controls the tradeoff between power and conservatism. Various predictable strategies exist, ranging from fixed values to adaptive schemes such as ONS and aGRAPA~\citep{waudbysmith2022estimatingmeansboundedrandom} that maximize expected log-wealth growth without prior knowledge of the signal distribution.
We instantiate this framework for chain-of-thought stopping by defining a chunk-level filtration and mapping trajectory signals to bounded e-value
increments via an EMA-deviation transform.
For more formal background, please refer to Appendix~\ref{sec:prelim}.

\section{Method}
\label{sec:method}

\subsection{Setup}
We consider a reasoning LLM that produces a chain-of-thought trace $r_{1:T} = (r_1, \ldots, r_T)$ of $T$ chunks for a query $x$ with gold answer $y$.
After observing $r_{1:t}$, the model yields a running prediction $f_t(x)$ and an auxiliary scalar $s_t \in \mathbb{R}$ that summarizes the per-step signal associated with $f_t$. The scalar $s_t$ may be an answer confidence, the stability of $f_t$ across recent steps, or any other primary signal used by an early-exit policy.
Prior work focuses on designing better primary signals \cite{fu2025deepconf, wang2025eat}; we instead study the early-exit policy that consumes them.

An early-exit policy selects a stopping time $\tau \le T$ and takes one of two actions: (i)~\emph{commit}, returning $f_\tau(x)$ when further reasoning is judged unnecessary; or (ii)~\emph{abstain}, deferring to a fallback when the instance appears unsolvable within the full budget. Both actions reduce token usage relative to running to $T$ while intermediate probing costs are explicitly accounted for in our evaluation.

\subsection{Goal}
Reasoning horizons and signal dynamics vary across models, datasets, and token budgets, so any rule tuned on a held-out set is unlikely to transfer. Models also encounter instances beyond their capacity, far from the validation distribution, where additional reasoning consumes tokens without improving the answer. The optimal threshold
likewise depends on these factors, so no fixed value reliably trades accuracy for efficiency.

We therefore design an early-stopping rule with three properties: (i)~no dependence on a held-out calibration set; (ii)~symmetric handling of commit and abstain under a single decision rule; and (iii)~user-specified control over the significance level $\alpha$.

\subsection{E-Process Construction}
Instead of asking whether a raw signal $s_t$ has crossed a calibrated threshold, we test whether the trajectory has accumulated sequential evidence in favor of stopping. To this end, we construct dual e-processes: $E^{+}_t$ accumulates evidence for commit, and $E^{-}_t$ accumulates evidence for abstain. Crossing the upper boundary indicates that further reasoning is redundant; crossing the lower boundary indicates that further reasoning is unlikely to improve the answer and triggers the fallback policy.

Each stream maps $s_t$ to a bounded evidence increment $\phi^{m}_t \in [-1, 1]$ for $m \in \{+, -\}$ and updates a nonnegative wealth process
\begin{equation}
E^{m}_t = \prod_{i=1}^{t} \bigl(1 + \lambda^{m}_i \phi^{m}_i\bigr),
\label{eq:savi_wealth_process}
\end{equation}
with $E^{m}_0 = 1$.
The betting fraction $\lambda^{m}_i \ge 0$ is predictable---it depends only on $\mathcal{F}_{i-1}$---and satisfies $1 + \lambda^{m}_i \phi^{m}_i \ge 0$ so that $E^{m}_t$ remains nonnegative. In practice, $\lambda^{m}_i$ can be fixed, selected from a small portfolio, or set adaptively (Section~\ref{sec:betting}).

We define a signal-level null hypothesis corresponding to the absence of systematic positive evidence drift by requiring the evidence increment to have nonpositive conditional mean,
\begin{equation}
H^{m}_0: \quad \mathbb{E}\!\bigl[\phi^{m}_t \,\big|\, \mathcal{F}_{t-1}\bigr] \le 0.
\end{equation}
Since $\lambda^{m}_t$ is $\mathcal{F}_{t-1}$-measurable, taking conditional expectation in Eq.~\ref{eq:savi_wealth_process} under $H^{m}_0$ gives
\begin{align}
\mathbb{E}\!\bigl[E^{m}_t \,\big|\, \mathcal{F}_{t-1}\bigr]
&= E^{m}_{t-1}\bigl(1 + \lambda^{m}_t\, \mathbb{E}[\phi^{m}_t \mid \mathcal{F}_{t-1}]\bigr) \nonumber\\
&\le E^{m}_{t-1},
\end{align}
so $E^{m}_t$ is a nonnegative supermartingale starting at $1$ under the null.

For any such process, Ville's inequality~\citep{ramdas2023game} bounds the probability that it ever exceeds a level $\alpha_m > 0$:
\begin{equation}
\Pr_{H^{m}_0}\!\left(\sup_{t} E^{m}_t \ge 1/\alpha_m\right) \le \alpha_m.
\end{equation}
Unlike Markov's inequality, the supremum is taken over the entire trajectory rather than at a fixed $t$, so the bound holds at \textit{every stopping time}, including those chosen adaptively from the current data. This anytime validity allows us to monitor $E^{m}_t$ continuously without inflating the false alarm rate. Then, we define the per-stream stopping time
\begin{equation}
\tau_m = \inf\{t : E^{m}_t \ge 1/\alpha_m\}.
\end{equation}
Stopping at $\tau_m$ therefore raises a false alarm with probability at most $\alpha_m$, uniformly in $t$.

\subsection{Position-Velocity Decomposition}
What remains is to choose the signal $s_t$ and its mapping to the bounded evidence $\phi^{m}_t$. Rather than designing a separate signal for each stream, we use a single $s_t$ and extract two complementary statistics: its position relative to a slow trend, and its velocity.

\paragraph{Commit increment (position).}
Let
\begin{equation}
    e_t = \gamma s_t + (1-\gamma)\, e_{t-1},
    \qquad 0 < \gamma \ll 1,
\end{equation}
denote a slow exponential moving average (EMA) of $s_t$. The commit increment is the clipped deviation from this trend,
\begin{equation}
    \phi^{+}_t = \mathrm{clip}(s_t - e_t,\, -1,\, 1).
    \label{eq:savi_commit_phi}
\end{equation}
The commit stream accumulates evidence when $s_t$ exceeds its slow trend, removing the need for an absolute threshold on $s_t$, a quantity whose scale shifts across models and datasets.

\paragraph{Abstain increment (velocity).}
The abstain increment is the signed velocity of $s_t$,
\begin{equation}
    \phi^{-}_t = \mathrm{sign}(s_{t-1} - s_t),
    \label{eq:savi_abstain_phi}
\end{equation}
positive when $s_t$ declines, negative when it rises, and zero otherwise. The abstain stream therefore accumulates evidence when the trajectory loses stability.

The two increments are anti-correlated by construction:
$\phi_t^+$ is non-decreasing in $s_t$, whereas $\phi_t^-$ is non-increasing in $s_t$; consequently, their bet-weighted increments are conditionally anti-correlated.

\subsection{Betting Rate}
\label{sec:betting}
Another design choice is the betting fraction $\lambda^{m}_t$. Two requirements constrain it: $\lambda^{m}_t$ must be predictable so that Ville's bound applies, and it should adapt to the realized drift of $\phi^{m}_t$---unknown a priori---without any held-out tuning. We adopt the aGRAPA~\citep{waudbysmith2022estimatingmeansboundedrandom},
\begin{equation}
\lambda^{m}_t = \mathrm{clip}\!\left(\frac{\sum_{s < t} \phi^{m}_s}{1 + \sum_{s < t} (\phi^{m}_s)^2},\; 0,\; \lambda_{\max}\right),
\label{eq:betting}
\end{equation}
with $\lambda_{\max} = 1/2$ for $1 + \lambda^{m}_t \phi^{m}_t \ge 1/2 > 0$. The rate depends only on $\phi^{m}_{<t}$, so it is $\mathcal{F}_{t-1}$-measurable. Under strong positive drift the numerator $\sum_{s < t} \phi^{m}_s$ grows, $\lambda^{m}_t$ saturates near $\lambda_{\max}$, and $E^{m}_t$ accelerates, while under weak drift $\lambda^{m}_t$ stays near zero and evidence accrues slowly.

\subsection{Dual E-Process Early Exits}
Combining the per-stream stopping times $\tau_+$ and $\tau_-$ with the budget $T$, the overall exit time is
\begin{equation}
    \tau = \min\{\tau_+,\, \tau_-,\, T\}.
\end{equation}
The action depends on which boundary triggers first: the policy commits to $f_\tau(x)$ if $\tau = \tau_+$, abstains if $\tau = \tau_-$, and returns the full-budget answer $f_T(x)$ if neither boundary is crossed before $T$.

\subsection{Sensitivity Control and Guarantee}

The significance levels $\alpha_+, \alpha_-$ are user-specified parameters that control the commit and abstain thresholds.
Each $\alpha_m$ acts as a monotone sensitivity knob.
A smaller $\alpha_+$ raises the threshold $1/\alpha_+$, delays $\tau_+$ on average, and yields fewer commits, more tokens, and (empirically) a lower task-level false-positive rate; the same monotonicity holds for $\alpha_-$ on the abstain stream.
The pair $(\alpha_+, \alpha_-)$ therefore controls the commit/abstain trade-off without calibration data; the formal guarantee is signal-level (see Appendix~\ref{sec:appen:theory}), and its empirical task-level behavior is reported in Section~\ref{sec:exp-sc}.

Two invariances make the $\alpha$-based control portable.
The bound is anytime-valid: it holds uniformly over the entire trajectory, so the rule can inspect $E_t^m$ at every chunk and stop at the first crossing without inflating the error rate.
The bound is also distribution-free, depending only on user-selected confidence level $\alpha_m$; combined with the local-trend normalization of $\phi_t^m$, the same $\alpha$ governs decisions across models, datasets, and token budgets without calibration data.

\begin{figure*}[t]
\centering
\includegraphics[width=0.99\textwidth]{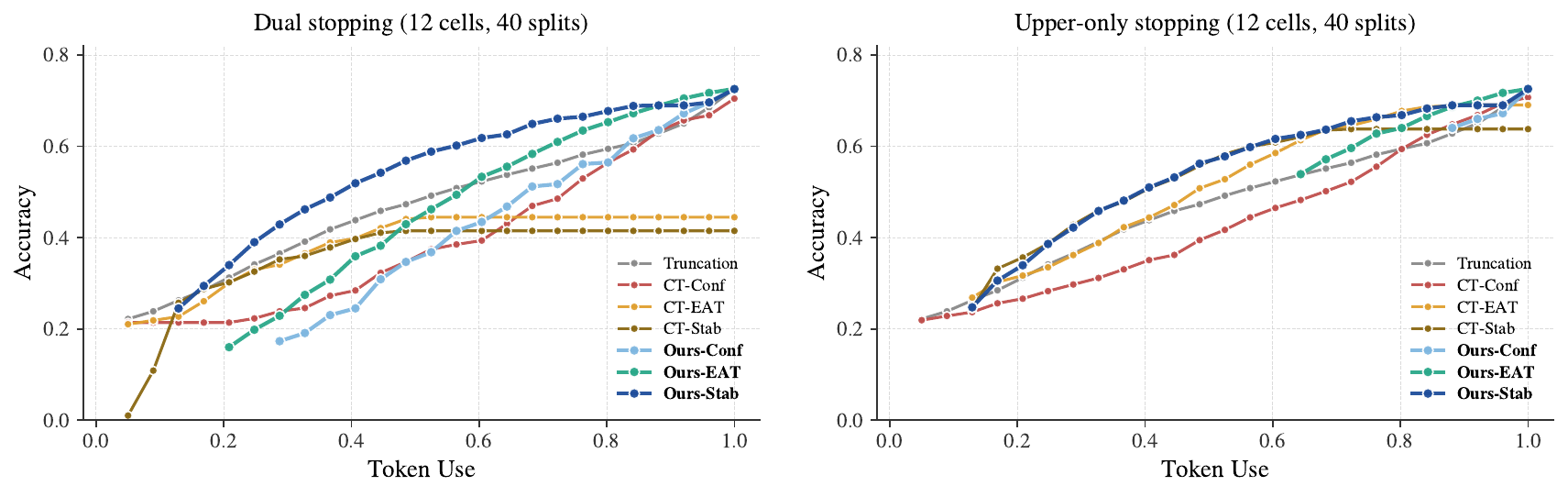}
\caption{Token-use vs.\ accuracy Pareto curves aggregated over all 12 model--dataset combinations. The left panel shows dual stopping, where early commit and early abstain are both active so the model can either commit to an intermediate answer or defer to a fallback policy before the trajectory ends. The right panel shows upper-only stopping, where only early commit is applied and the model falls back to the terminal answer otherwise.}
\label{fig:fig5}
\end{figure*}

\begin{table*}[t]
\small\centering
\resizebox{\textwidth}{!}{
\renewcommand{\arraystretch}{1.2}
\setlength{\tabcolsep}{3pt}
\begin{tabular}{l|ccc|ccc|ccc|ccc}
\toprule
& \multicolumn{6}{c|}{\textbf{Dual stopping}} & \multicolumn{6}{c}{\textbf{Upper-only stopping}} \\
\textbf{TokU} & \textbf{CT-Conf} & \textbf{CT-EAT} & \textbf{CT-Stab} & \textbf{Ours-Conf} & \textbf{Ours-EAT} & \textbf{Ours-Stab} & \textbf{CT-Conf} & \textbf{CT-EAT} & \textbf{CT-Stab} & \textbf{Ours-Conf} & \textbf{Ours-EAT} & \textbf{Ours-Stab} \\
\midrule
$\le 0.2$ & 0.214 & 0.286 & 0.302 & 0.131 & 0.158 & \textbf{0.337} & 0.275 & 0.286 & 0.310 & ---   & ---   & \textbf{0.317} \\
$\le 0.3$ & 0.243 & 0.347 & 0.352 & 0.180 & 0.235 & \textbf{0.441} & 0.310 & 0.364 & \textbf{0.408} & ---   & ---   & \textbf{0.408} \\
$\le 0.4$ & 0.280 & 0.398 & 0.397 & 0.245 & 0.358 & \textbf{0.517} & 0.363 & 0.453 & 0.485 & ---   & ---   & \textbf{0.491} \\
$\le 0.5$ & 0.357 & 0.441 & 0.415 & 0.364 & 0.444 & \textbf{0.574} & 0.408 & 0.545 & 0.538 & ---   & ---   & \textbf{0.567} \\
$\le 0.6$ & 0.394 & 0.445 & 0.415 & 0.434 & 0.530 & \textbf{0.616} & 0.502 & 0.595 & 0.614 & ---   & 0.508 & \textbf{0.637} \\
$\le 0.7$ & 0.486 & 0.445 & 0.415 & 0.512 & 0.594 & \textbf{0.656} & 0.557 & 0.628 & 0.631 & ---   & 0.591 & \textbf{0.637} \\
$\le 0.8$ & 0.562 & 0.445 & 0.415 & 0.564 & 0.653 & \textbf{0.677} & 0.589 & 0.683 & 0.641 & 0.611 & 0.636 & \textbf{0.690} \\
$\le 0.9$ & 0.645 & 0.445 & 0.415 & 0.662 & \textbf{0.697} & 0.690 & 0.654 & 0.683 & 0.641 & 0.670 & \textbf{0.700} & 0.690 \\
$\le 1.0$ & 0.705 & 0.445 & 0.415 & \textbf{0.726} & \textbf{0.726} & \textbf{0.726} & 0.707 & 0.683 & 0.641 & \textbf{0.726} & \textbf{0.726} & \textbf{0.726} \\
\bottomrule
\end{tabular}
}
\caption{Accuracy at each token budget, averaged over 40 splits and 12 cells. CT methods use $n_{\mathrm{val}}=50$ while Ours uses zero calibration data. CT-Stab applies CT's calibration rule to the same stability signal used by Ours. {---}: no Pareto operating point exists at or below that budget (the e-process commit threshold was not crossed at any $\alpha$ in the sweep). \textbf{Bold}: highest value in each half-row.}
\label{tab:fig5}
\end{table*}

\section{Experiments}

\subsection{Experimental Setup}

\paragraph{Models and datasets.}
We apply our method to three reasoning LLMs: Qwen3-8B, Qwen3-30B-A3B-Instruct-2507~\citep{yang2025qwen3}, and DeepSeek-R1-Distill-Qwen-32B~\citep{guo2025deepseek}. We evaluate across four mathematical reasoning benchmarks: AIME, MATH-500~\citep{lightman2023verify}, OlympiadBench~\citep{he-etal-2024-olympiadbench}, and DeepScaler~\citep{luo2025deepscaler}. AIME comprises 1,011 competition problems drawn from the American Invitational Mathematics Examination spanning 1983 to 2024. MATH-500 is a 500-problem subset of the MATH benchmark covering a broad range of high-school competition topics. OlympiadBench is a multilingual olympiad-level benchmark, from which we use the English math subset of 674 problems. DeepScaler is a curated collection of 1,189 challenging math problems designed to stress-test advanced reasoning models. Token use includes the reasoning prefix together with every forcing-suffix prefill and intermediate-answer decode performed up to the stopping boundary. CT-EAT and CT-Stab incur the same probes as their CES counterparts, while truncation and confidence-only methods do not require answer probes. Position-aware FLOPs accounting and detailed overhead statistics are reported in the supplementary material.

\paragraph{Trace collection.}
Each problem is presented to the model with the instruction to reason step by step, separating reasoning steps by blank lines, and to write the final answer inside $\backslash$\texttt{boxed\{\}}. Trajectories are generated using vLLM~\citep{kwon2023efficient} following the official sampling parameters recommended by each model's authors; for Qwen3 models we use $T{=}0.6$, $p{=}0.95$, $k{=}20$ (with $T{=}0.7$, $p{=}0.8$ for the Qwen3-30B-A3B variant), and for DeepSeek-R1 we use $T{=}0.6$, $p{=}0.95$. The output text is segmented into reasoning chunks at blank-line boundaries.

\paragraph{Intermediate predictions and signals.}
At each chunk boundary, we append a forcing suffix (\texttt{\textbackslash n\textbackslash n**Final Answer**\textbackslash n\textbackslash boxed\{}) to the partial trace and complete greedily ($T{=}0$, ${\le}24$ tokens) to extract an intermediate boxed prediction $\hat{y}_t$. Three scalar signals are derived from the partial trace at each boundary: Confidence~\citep{yang2025deer} (geometric mean of per-token probabilities), EAT~\citep{wang2025eat} (inverted first-answer-token entropy), and Stability (normalized run length of consecutive equal predictions). Each signal is smoothed with a causal EMA ($\alpha{=}0.3$) before entering the e-process; the bet-construction EMA uses $\gamma{=}0.05$. Formal definitions are in Appendix~\ref{sec:signal}.

\paragraph{Protocol.}
We follow CT's~\citep{wang2026conformal} evaluation protocol: 40 random splits, validation size $n_{\mathrm{val}} = 50$, test capped at 800. For ours, the validation set is unused at deployment; for CT, it is used to UCB-calibrate the upper threshold $\tau$ (and the lower sigmoid steepness $c$ in the dual variant). Sweeps cover $\alpha \in \{10^{-4}, \ldots, 0.95\}$ for ours and target risk $\eps \in \{0.05, 0.10, \ldots, 1.00\}$ for CT.

\begin{figure*}[t]
\centering
\begin{subfigure}[t]{0.49\textwidth}
\includegraphics[width=\textwidth]{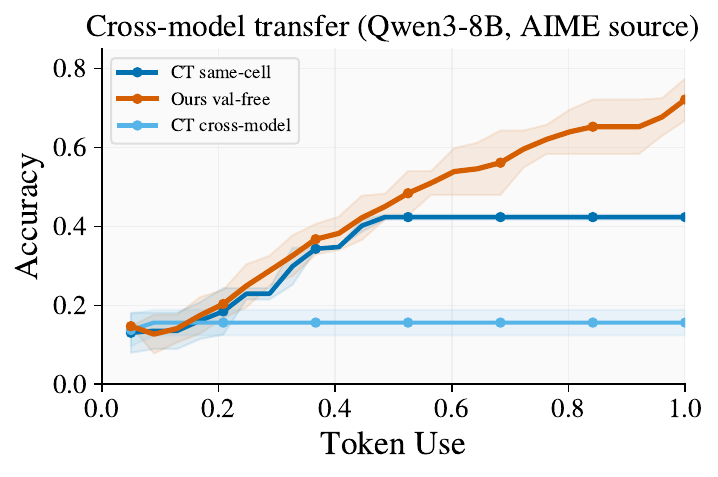}
\caption{Cross-model: different models on AIME.}
\end{subfigure}
\hfill
\begin{subfigure}[t]{0.49\textwidth}
\includegraphics[width=\textwidth]{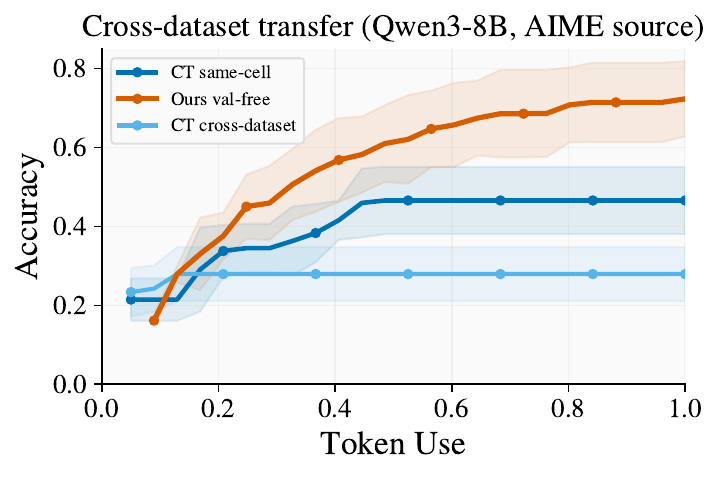}
\caption{Cross-dataset: Qwen3-8B on other datasets.}
\end{subfigure}
\caption{Token-use vs.\ accuracy Pareto under cross-combination deployment with Qwen3-8B/AIME as the fixed calibration source. CT same-cell and the proposed method are evaluated on the same target cells as reference envelopes. CT cross-cell performance degrades substantially when the calibration distribution does not match the target cell, whereas the proposed method operates identically without re-calibration.}
\label{fig:cross-cell}
\end{figure*}

\subsection{Experiment Results}

We evaluate the proposed method against CT and a truncation baseline across all 12 model--dataset combinations (3 models $\times$ 4 datasets), reporting accuracy--token-use Pareto envelopes aggregated over all cells, for both upper-only stopping and dual stopping settings. Each curve sweeps over the method's parameter grid. Alongside our main recipe (stability signal), we also plug the signals evaluated in CT (EAT~\citep{wang2025eat} and confidence~\citep{yang2025deer}) directly into our e-process framework, and conversely plug the stability signal into CT's framework, allowing a controlled signal--framework comparison.

As shown in Figure~\ref{fig:fig5} and Table~\ref{tab:fig5}, the proposed method with stability signal clearly dominates CT across the majority of the token budget range in both settings, despite using zero calibration data where CT uses 50 validation trajectories per split. The advantage is consistent rather than concentrated at a single operating point. Using the signals evaluated in CT inside our e-process underperforms relative to the stability signal. This gap arises because EAT and confidence signals tend to plateau near the upper end of $[0,1]$ once a reasoning chain is on track, producing a conditional mean that is approximately stationary under $\Hzu$. Such stationarity is well-suited to CT's sigmoid threshold, which compares the signal against a fixed level at each step. In our e-process, the wealth increment at each step measures how much the current signal deviates from its recent exponential moving average. When the signal is stationary, these deviations are near zero, the wealth accumulates negligibly, and the e-process rarely accumulates enough wealth to reach its crossing threshold. The stability signal avoids this problem by construction: it tracks changes in prediction consistency across chunks, so it continues to emit informative increments even late in the chain when EAT and confidence have already flattened. The truncation baseline, which stops generation at a fixed fraction of the token budget without any signal, is consistently dominated by the proposed method across the practical operating range, confirming that the accuracy gains stem from evidence-driven stopping rather than naive budget truncation. Per-combination Pareto curves are provided in Figure~\ref{fig:fig5-per-cell} in Appendix~\ref{appx:exp}.

\subsection{Cross-Combination Deployment}

We evaluate cross-cell generalization by fixing Qwen3-8B/AIME as the calibration source and deploying to all other model--dataset combinations. As shown in Figure~\ref{fig:cross-cell}, CT requires fitting a threshold $\tau$ on a held-out calibration set drawn from the deployment distribution, so reusing a source-calibrated $\tau$ on a new target combination breaks the calibration bound. The transferred $\tau$ is systematically miscalibrated: in the low-risk regime calibration is often infeasible under $n_{\mathrm{val}}=50$, and at higher risks the threshold is over-conservative, consuming more tokens than necessary.

The proposed method requires no calibration data at any stage: the commit and abstain decisions depend only on the current trajectory's e-process wealth, so the same $\alpha$ applies identically to any target cell without re-calibration. As a result, the proposed method achieves a consistently better token--accuracy trade-off across the full operating range in both cross-model and cross-dataset settings.

\begin{figure*}[t]
\centering
\begin{subfigure}[t]{0.71\textwidth}
  \includegraphics[width=\textwidth]{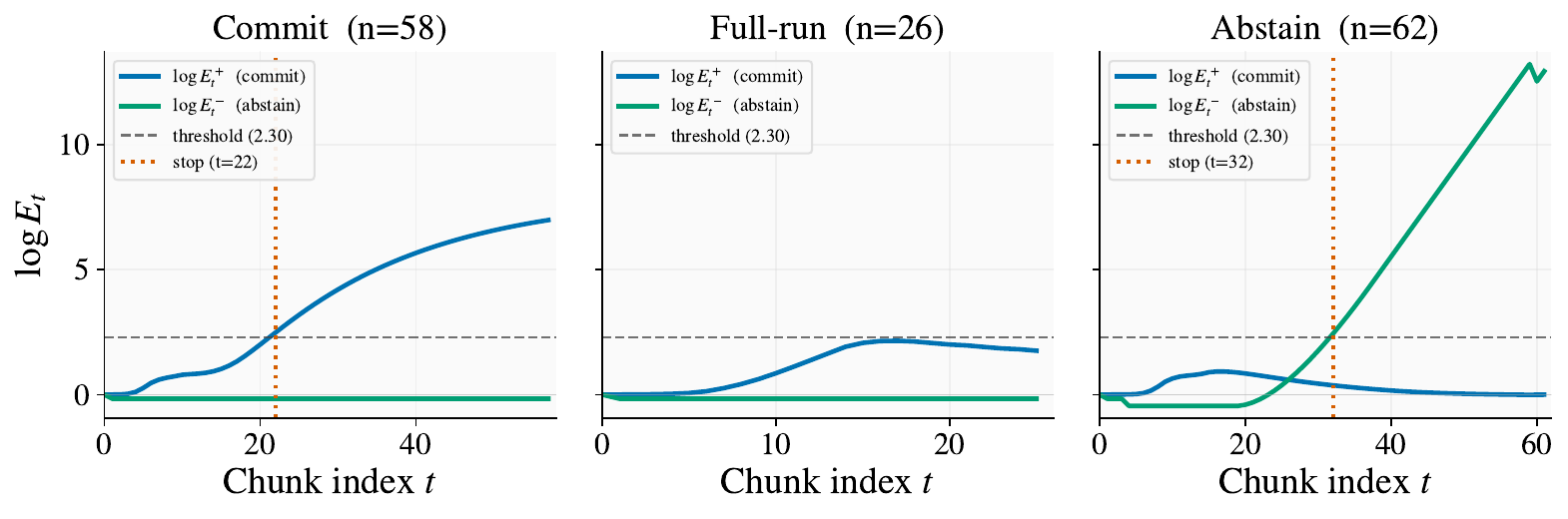}
  \caption{Trajectory outcomes ($\alpha{=}0.1$).}\label{fig:traj-examples}
\end{subfigure}
\hfill
\begin{subfigure}[t]{0.26\textwidth}
  \includegraphics[width=\textwidth]{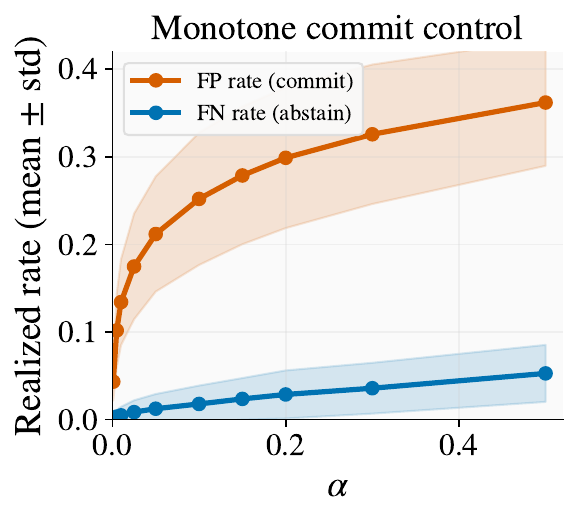}
  \caption{FP/FN rates vs.\ $\alpha$.}\label{fig:sc3}
\end{subfigure}
\caption{E-process dynamics. \textbf{(a)} Three canonical trajectory outcomes at $\alpha{=}0.1$: $\log\Eu_t$ rises and crosses the threshold for a correct early prediction (\emph{commit}); neither stream crosses for a hard instance that exhausts the full token budget (\emph{full-run}); $\log\El_t$ crosses first, correctly flagging a trajectory that would have produced a wrong terminal answer (\emph{abstain}). \textbf{(b)} Realized commit FP rate and abstain FN rate as $\alpha$ is swept across all 12 model--dataset combinations; both decrease monotonically, confirming that $(\alphap, \alpham)$ provide calibration-free monotone knobs on the two error types.}
\end{figure*}

\begin{figure*}[t]
\centering
\includegraphics[width=0.98\textwidth]{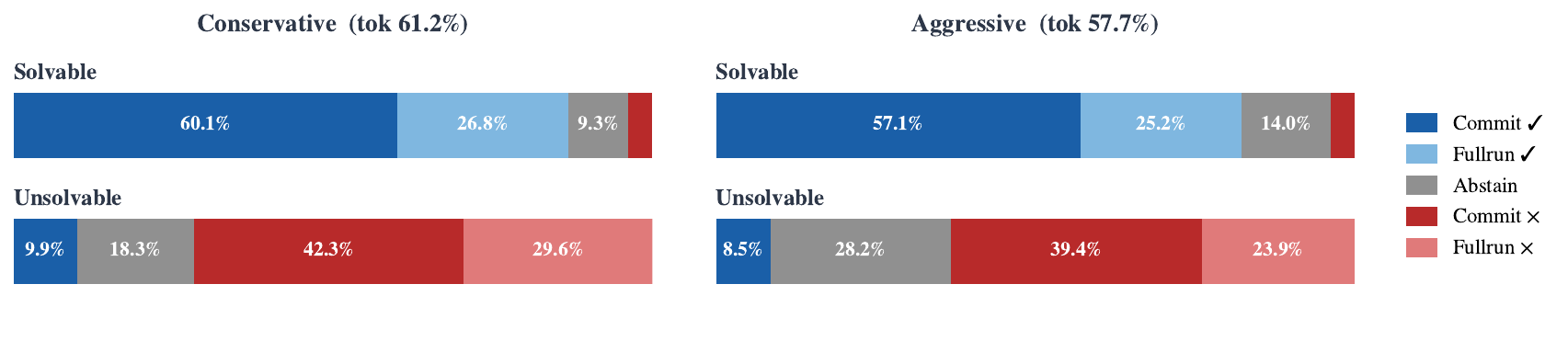}
\caption{Stopping outcome distribution for Qwen3-8B on MATH-500 under two operating points. Each bar shows the fraction of trajectories in each
  outcome category, further decomposed into correct and incorrect predictions, broken down by problem solvability.}
\label{fig:stopping-ratio}
\end{figure*}

\subsection{E-process Dynamics}
\label{sec:exp-sc}

\paragraph{Trajectory outcomes.}
We visualize representative trajectories to illustrate how the two wealth processes evolve across the three possible stopping outcomes. As shown in Figure~\ref{fig:traj-examples}, we present three cases at $\alpha{=}0.1$. In the commit case, $\log\Eu_t$ rises steadily as the intermediate answers stabilize and crosses the threshold well before the final chunk, yielding an early correct prediction. In the full-run case, neither wealth process crosses the threshold: the model uses all available tokens and returns the terminal answer. In the abstain case, $\log\El_t$ crosses the threshold while $\log\Eu_t$ remains low throughout, flagging a trajectory on which the stability signal fails to converge.

\paragraph{Monotone commit control.}
Figure~\ref{fig:sc3} plots the realized FP rate (commit stream) and FN rate (abstain stream) as $\alpha$ is swept across all 12 model--dataset combinations. Both rates decrease monotonically as $\alpha$ decreases, confirming that $\alphap$ and $\alpham$ provide independently monotone knobs on the two error types without requiring calibration data.

\paragraph{Stopping outcome distribution.}

We analyze the distribution of trajectory outcomes at two operating points to characterize how the commit and abstain streams behave in practice. As shown in Figure~\ref{fig:stopping-ratio}, we report results for Qwen3-8B on MATH-500 under conservative and aggressive parameter settings, decomposing each outcome into whether it results in a correct or incorrect prediction. Early commit dominates in both settings, and the large majority of commits on solvable problems are correct, confirming that compute is reclaimed without sacrificing accuracy. On unsolvable problems, early commit captures a non-trivial fraction of correct predictions that would otherwise be lost to overthinking under a full-run policy, suggesting that the upper e-process fires precisely in windows where the model has already reached a correct intermediate answer before the reasoning chain deteriorates. The abstain stream exhibits a desirable asymmetry: unsolvable problems abstain at a substantially higher rate than solvable ones, indicating that the lower e-process selectively identifies likely-incorrect trajectories rather than firing indiscriminately. Moving from the conservative to the aggressive operating point raises the unsolvable-abstain rate while reducing mean token use, at the cost of a modest increase in false abstentions on solvable problems. Full-run cases --- where neither e-process reaches its threshold within the token budget --- account for a minority of trajectories in both settings.

\subsection{Ablation Study}

\begin{figure}[t]
\centering
\includegraphics[width=0.46\textwidth]{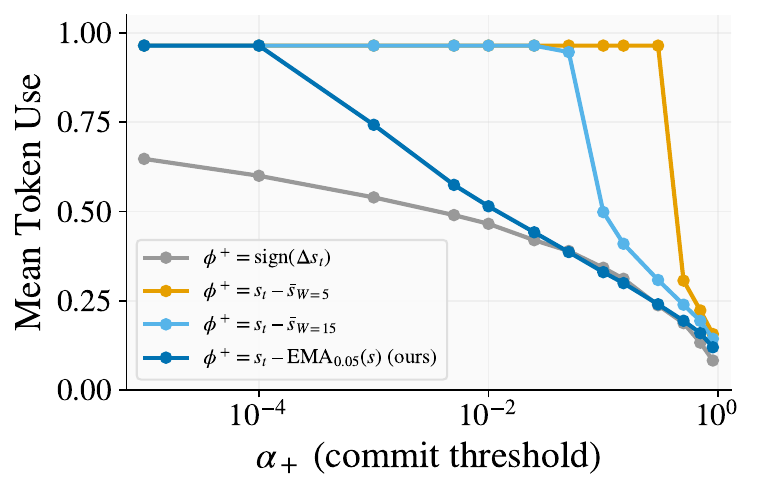}
\caption{Four choices of the upper evidence increment $\phi^+$ are compared by measuring mean token use as a function of $\alpha_+$ (log scale), with the lower stream fixed to $\phi^- = \mathrm{sign}(-\Delta s_t)$ and results aggregated over all 12 model--dataset combinations.}
\label{fig:ablation-phi}
\end{figure}

\paragraph{$\phi^+$ construction.}
We ablate the choice of upper $\phi^+$ by evaluating four constructions with the stability signal held fixed throughout. As shown in Figure~\ref{fig:ablation-phi}, we report how mean token use responds to $\alpha_+$ for each choice. Sign ($\phi^+ = \mathrm{sign}(\Delta s_t)$) responds to every $\alpha_+$ value but saturates in the high-budget regime: the binary increment fires too readily, leaving no operating points there. Window-mean baselines ($\bar{s}_{W=5}$, $\bar{s}_{W=15}$) exhibit near-zero $\alpha_+$ sensitivity across most of the practical range because the short window adapts quickly to the signal level, keeping $s_t - \bar{s}_W$ small and preventing wealth accumulation. EMA$_{0.05}$ decreases monotonically across the full $\alpha_+$ range, providing fine-grained budget control at every point.

\paragraph{$\lambda$ update rule.}
We ablate the choice of betting fraction $\lambda$ by comparing five strategies with $\phi^+$ fixed to EMA$_{0.05}$: three fixed values ($\lambda \in \{0.05, 0.15, 0.50\}$), ONS-m, and aGRAPA. As shown in Figure~\ref{fig:ablation-lambda}, fixed $\lambda$ performance varies by value: $\lambda{=}0.05$ and $\lambda{=}0.15$ under-commit and fall below the adaptive frontier, while $\lambda{=}0.50$ performs comparably to aGRAPA but requires grid-searching over the $\lambda$ space. ONS-m~\citep{pmlr-v75-cutkosky18a, waudbysmith2022estimatingmeansboundedrandom}, the Online Newton Step algorithm optimized for long observation horizons, collapses its $\lambda$ to zero within the first few chunks of each reasoning trace, yielding operating points only near full token use and providing no early-stopping benefit; this shows that algorithms calibrated for long sequences degenerate when the observation horizon is as short as a reasoning trace. aGRAPA sustains non-zero $\lambda$ throughout each trace and achieves a Pareto area matching the best fixed $\lambda$ ($\lambda{=}0.50$) without requiring any hyperparameter search over the $\lambda$ grid, making it our default choice.

  \begin{figure}[t]
\centering
\includegraphics[width=0.48\textwidth]{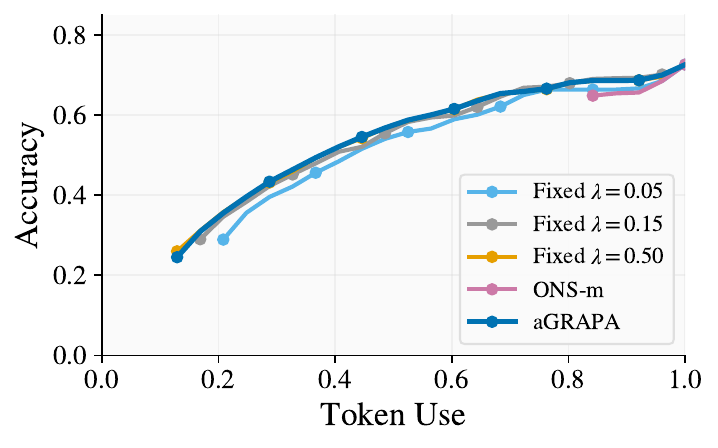}
\caption{Five $\lambda$ update strategies are compared on the token use--accuracy Pareto curve, with $\phi^+$ fixed to EMA$_{0.05}$ and results aggregated over all 12 model--dataset combinations.}
\label{fig:ablation-lambda}
\end{figure}

\section{Conclusion}
We propose CES, a calibration-free framework for adaptive chain-of-thought stopping. CES accumulates sequential evidence through two parallel e-processes on a shared signal decomposed into position and signed velocity; the commit and abstain streams are anti-correlated by construction, so evidence for one decision is structurally evidence against the other. Under the null that stopping is unsafe, Ville's inequality gives an anytime-valid false-alarm bound controlled by a single pair $(\alpha_+, \alpha_-)$. The bound holds at every step and under arbitrary continuation policies.
The construction removes the calibration set from the deployment loop. Calibration-based stoppers tie their guarantees to a specific (model, dataset, budget) configuration and must be re-tuned whenever any factor shifts; with unseen models or distribution drift, calibration data is unavailable or stale. CES carries no such dependence: the same $\alpha$ transfers across models, datasets, and token budgets without held-out data.

\paragraph{Future work.}

Future work can consider composing the signal-level e-process with a downstream conformal layer to obtain a task-level coverage guarantee on top of CES; instance-wise setting of $\alpha$ via test-time learning; extending the framework to alternative forms of test-time scaling, e.g., parallel reasoning chains; and improving the underlying per-chunk signal itself, e.g., by adding an auxiliary objective during post-training so that the increments $\phi^m_t$ correlate more directly with answer correctness.

\section*{Limitations}
The primary limitation of our approach is that the false-alarm guarantee is signal-level: under the null of no systematic drift in $\phi^m_t$, Ville's inequality bounds the probability that $E^m_t$ crosses $1/\alpha_m$, but not the probability that an early commit yields a wrong answer. The correspondence between $\alpha_m$ and the task-level FP rate is monotone (Section~\ref{sec:exp-sc}) but not numerically tight, with the looseness depending on the trajectory distribution and on the correlation between $s_t$ and answer correctness.
Applications targeting a strict task-level risk should therefore layer a downstream calibration step on top of our framework.

\section*{Acknowledgements}
This work was partly supported by the InnoCORE program of the Ministry of Science and ICT (AI Meta-Scientist, N10260110) and Artificial intelligence industrial convergence cluster development project funded by the Ministry of Science and ICT(MSIT, Korea)\&Gwangju Metropolitan City.

\bibliography{references}

\appendix

\setcounter{table}{0}
\renewcommand{\thetable}{\Alph{table}}
\setcounter{figure}{0}
\renewcommand{\thefigure}{\Alph{figure}}
\setcounter{equation}{0}
\renewcommand{\theequation}{\Alph{equation}}
\setcounter{section}{0}
\renewcommand{\thesection}{\Alph{section}}

\newpage
\onecolumn

\section{Signal Definitions}
\label{sec:signal}

We evaluate three per-chunk scalar signals $s_t \in [0,1]$, each smoothed after raw computation with an exponential moving average (EMA) at $\alpha = 0.3$:
\[
\bar{s}_t = 0.3 \cdot s_t + 0.7 \cdot \bar{s}_{t-1}, \quad \bar{s}_0 = s_0.
\]

\paragraph{Confidence.}
The confidence signal is the geometric mean of per-token generation probabilities over chunk $r_t$:
\[
s_t = \exp\!\left(\frac{1}{|r_t|}\sum_{i \in r_t} \log p_\theta(x_i \mid x_{<i})\right) \in (0, 1],
\]
where $|r_t|$ is the number of tokens in the chunk. High values indicate that the model generates each token with high probability; near-zero values indicate uncertain generation. This signal is adopted from~\citet{yang2025deer}.

\paragraph{EAT (Emission Answer Token).}
At each chunk boundary the model is forced to emit an intermediate answer; EAT records the entropy of the model distribution over the \emph{first} emitted answer token~\citep{wang2025eat}. We invert and normalize to $[0,1]$ so that higher values indicate higher confidence:
\[
s_t = 1 - \min\!\left(1,\; \frac{H_t}{\log 10}\right), \quad H_t = -\sum_v p_\theta(v \mid x, r_{1:t})\log p_\theta(v \mid x, r_{1:t}),
\]
where the sum is over the vocabulary and $\log 10$ normalizes against a ten-class uniform baseline.

\paragraph{Stability.}
The stability signal measures the persistence of the intermediate answer across consecutive chunks. Let $f_t(x)$ denote the normalized intermediate answer at chunk $t$. Define the run length as the number of consecutive chunks ending at $t$ whose normalized answers all agree:
\[
\mathrm{run}_t = \max\bigl\{k \ge 1 : f_{t-k+1}(x) = \cdots = f_t(x)\bigr\}.
\]
The raw signal is
\[
s_t = \frac{\min(\mathrm{run}_t,\; 8)}{8} \in [0, 1],
\]
capped at 8 to bound sensitivity to long uninterrupted runs. High stability indicates the model has converged on an answer; low stability indicates the answer is still changing.

\section{Deeper Background}

\label{sec:prelim}

\subsection{Setup}
Let $(\F_t)_{t \ge 0}$ be a filtration with $\F_0$ trivial. A process $(X_t)$ is \emph{predictable} if $X_t$ is $\F_{t-1}$-measurable for all $t \ge 1$. We work under a null hypothesis $H_0$, possibly composite, which constrains the conditional mean of the increment process specified in Eq.~\eqref{eq:null-mean}. Expectation $\E[\cdot]$ and probability $\Pr(\cdot)$ are taken under $H_0$ throughout.

\subsection{E-processes and test supermartingales}
A nonnegative process $(E_t)_{t \ge 0}$ adapted to $(\F_t)$ with $E_0 = 1$ is an \emph{e-process} for $H_0$ if
\begin{equation}
\E[E_\tau] \le 1
\end{equation}
for every stopping time $\tau$. A \emph{test supermartingale} under $H_0$ is a nonnegative process $(M_t)$ with $M_0 = 1$ and $\E[M_t \mid \F_{t-1}] \le M_{t-1}$. Every test supermartingale is an e-process; the converse holds for the constructions used here.

\subsection{Ville's inequality}
For any e-process $(E_t)$ and any $\alpha \in (0, 1)$,
\begin{equation}
\Pr\!\left( \sup_{t \ge 0} E_t \ge 1/\alpha \right) \le \alpha.
\end{equation}
The bound holds uniformly over $t$. Consequently, the stopping rule
\begin{equation}
\tau_\alpha = \inf\{\, t \ge 0 : E_t \ge 1/\alpha \,\}
\end{equation}
controls the type-I error at level $\alpha$ regardless of when or how often $E_t$ is inspected. This \emph{anytime-validity} distinguishes e-processes from $p$-values, which require a pre-specified sample size.

\subsection{Wealth construction by betting}
Let $(\phi_t)$ be an $\F_t$-adapted increment process with $|\phi_t| \le 1$ and
\begin{equation}
\E[\phi_t \mid \F_{t-1}] \le 0, \qquad t \ge 1.
\label{eq:null-mean}
\end{equation}
For any predictable $(\lambda_t)$ with $\lambda_t \ge 0$ and $1 + \lambda_t \phi_t > 0$ a.s., the wealth process
\begin{equation}
E_t = \prod_{s=1}^{t} \bigl(1 + \lambda_s\, \phi_s\bigr), \qquad E_0 = 1,
\end{equation}
is a nonnegative test supermartingale, and hence an e-process, for any $H_0$ satisfying \eqref{eq:null-mean}. The interpretation is a gambler who, at round $t$, stakes a fraction $\lambda_t$ of current wealth on the event $\{\phi_t > 0\}$.

\subsection{aGRAPA Betting Fraction}
The bet $\lambda_t$ governs the power of the test: small $\lambda_t$ accumulates evidence slowly, large $\lambda_t$ risks rapid wealth loss when $\phi_t < 0$. A standard adaptive choice with no tuning parameter is aGRAPA \citep{waudbysmith2022estimatingmeansboundedrandom}
\begin{equation}
\lambda_t = \clip\!\left(\frac{\sum_{s<t} \phi_s}{1 + \sum_{s<t} \phi_s^2},\ 0,\ \lambda_{\max}\right),
\end{equation}
which is $\F_{t-1}$-measurable and approximates the Kelly-optimal fraction online. With $|\phi_t| \le 1$ and $\lambda_{\max} \le 1/2$, the nonnegativity constraint $1 + \lambda_t \phi_t > 0$ holds automatically.

\subsection{Composite nulls}
The construction requires only the one-sided constraint \eqref{eq:null-mean} on the conditional mean of $\phi_t$, not a fully specified distribution. The resulting e-process is therefore valid for any $H_0$ under which \eqref{eq:null-mean} holds, which is essential when the null is non-parametric or distributional assumptions on the underlying signal are weak.

\section{Theoretical Properties}
\label{sec:appen:theory}

This section establishes the anytime-valid guarantees of the
construction. Notation follows Section~\ref{sec:method};
$m \in \{+, -\}$ indexes the two streams.

\begin{lemma}[Bounded increments and predictable bets]
\label{lem:bounded}
For each $t \ge 1$:
\emph{(i)} $\phi^+_t \in [-1, 1]$ and $\phi^-_t \in \{-1, 0, 1\}$;
\emph{(ii)} $\lambda^m_t$ is $\mathcal{F}_{t-1}$-measurable with
$\lambda^m_t \in [0, \lambda_{\max}]$ and $\lambda_{\max} = 1/2$;
\emph{(iii)} $1 + \lambda^m_t \phi^m_t \ge 1/2$ almost surely.
\end{lemma}
\begin{proof}
\emph{(i)} and \emph{(ii)} follow from
\eqref{eq:savi_commit_phi}--\eqref{eq:betting}. For \emph{(iii)},
$\lambda^m_t \phi^m_t \ge -\lambda^m_t \ge -1/2$.
\end{proof}

\begin{definition}[Signal-level nulls]
\begin{equation*}
H_0^+ : \mathbb{E}[\phi^+_t \mid \mathcal{F}_{t-1}] \le 0,
\qquad
H_0^- : \mathbb{E}[\phi^-_t \mid \mathcal{F}_{t-1}] \le 0,
\qquad \forall t \ge 1.
\end{equation*}
$H_0^+$ asserts that $s_t$ does not systematically exceed its slow
baseline; $H_0^-$ asserts that $s_t$ does not systematically decrease.
\end{definition}

\begin{lemma}[Supermartingale property]
\label{lem:supermart}
Under $H_0^m$, $(E^m_t)_{t \ge 0}$ is a nonnegative
$(\mathcal{F}_t)$-supermartingale with $E^m_0 = 1$.
\end{lemma}
\begin{proof}
Nonnegativity follows from Lemma~\ref{lem:bounded}\,(iii) by induction.
Since $\lambda^m_t \ge 0$ is $\mathcal{F}_{t-1}$-measurable,
\begin{equation*}
\mathbb{E}[E^m_t \mid \mathcal{F}_{t-1}]
= E^m_{t-1}\bigl(1 + \lambda^m_t\,
   \mathbb{E}[\phi^m_t \mid \mathcal{F}_{t-1}]\bigr)
\le E^m_{t-1}
\end{equation*}
under $H_0^m$.
\end{proof}

\begin{theorem}[Marginal type-I under the strict null]
\label{thm:ville}
For any $\alpha_m \in (0, 1)$,
\begin{equation}
\Pr_{H_0^m}\!\bigl(\tau_m \le T\bigr) \le \alpha_m.
\label{eq:ville-bound}
\end{equation}
\end{theorem}
\begin{proof}
By Lemma~\ref{lem:supermart} and Ville's inequality,
$\Pr_{H_0^m}(\sup_{t \le T} E^m_t \ge 1/\alpha_m) \le \alpha_m$,
which equals $\Pr_{H_0^m}(\tau_m \le T)$ by definition of $\tau_m$.
\end{proof}

The strict null requires zero mean drift of $\phi^m_t$. In practice
$e_t$ lags $s_t$, so a slowly drifting signal can violate the
assumption. We therefore state a relaxed bound that allows $E^m_t$ to
be a submartingale.

\begin{theorem}[Marginal type-I under positive drift]
\label{thm:doob}
Suppose $(E^m_t)_{t=0}^T$ is a nonnegative submartingale with
$\mathbb{E}[E^m_T] \le c_m$ for some $c_m \ge 1$. Then for any
$\alpha_m \in (0, 1)$,
\begin{equation}
\Pr\!\bigl(\tau_m \le T\bigr) \le \alpha_m \, c_m.
\label{eq:doob-bound}
\end{equation}
\end{theorem}
\begin{proof}
Doob's maximal inequality gives, for any $a > 0$,
$\Pr(\sup_{t \le T} E^m_t \ge a) \le \mathbb{E}[E^m_T]/a \le c_m/a$.
Set $a = 1/\alpha_m$.
\end{proof}

The constants $c_+, c_-$ are reported empirically in
Table~\ref{tab:per_cell_errors}; $\alpha_m$ remains monotone in the
false-decision probability.

\begin{proposition}[Structural anti-correlation]
\label{prop:anticorr}
For every $t \ge 1$,
\begin{equation}
\mathrm{Cov}\!\bigl(\lambda^+_t \phi^+_t,\ \lambda^-_t \phi^-_t
 \,\big|\, \mathcal{F}_{t-1}\bigr) \le 0.
\end{equation}
\end{proposition}
\begin{proof}
Conditional on $\mathcal{F}_{t-1}$, the quantities $s_{t-1}, e_{t-1},
\lambda^+_t, \lambda^-_t$ are constants with $\lambda^m_t \ge 0$.
Since $e_t = \gamma s_t + (1-\gamma) e_{t-1}$, we have
$s_t - e_t = (1-\gamma)(s_t - e_{t-1})$, which is non-decreasing in
$s_t$; hence $\phi^+_t = \mathrm{clip}(s_t - e_t, -1, 1)$ is
non-decreasing in $s_t$. By \eqref{eq:savi_abstain_phi},
$\phi^-_t = \mathrm{sign}(s_{t-1} - s_t)$ is non-increasing in $s_t$.
Multiplying by the nonnegative constants $\lambda^+_t, \lambda^-_t$
preserves monotonicity. The conditional covariance of two oppositely
monotone functions of a single random variable is nonpositive
(Harris's inequality, univariate case).
\end{proof}

\begin{theorem}[Joint type-I via union]
\label{thm:joint}
Under $H_0 := H_0^+ \cap H_0^-$,
\begin{equation}
\Pr_{H_0}\!\bigl(\{\tau_+ \le T\} \cup \{\tau_- \le T\}\bigr)
\le \alpha_+ + \alpha_-.
\label{eq:joint-bound}
\end{equation}
Under the relaxed null of Theorem~\ref{thm:doob}, the bound becomes
$\alpha_+ c_+ + \alpha_- c_-$.
\end{theorem}
\begin{proof}
The false early-decision event is
$\{\tau_+ \le T\} \cup \{\tau_- \le T\}$. Apply the union bound to
\eqref{eq:ville-bound} or \eqref{eq:doob-bound}.
\end{proof}

By Proposition~\ref{prop:anticorr}, the per-step increments composing
$E^+_t$ and $E^-_t$ are conditionally anti-correlated, so simultaneous
crossings are suppressed and \eqref{eq:joint-bound} is typically loose
in practice.

\paragraph{Scope of the guarantee.}
Theorems~\ref{thm:ville}--\ref{thm:joint} bound the wealth-crossing
probability under signal-level nulls, i.e., the absence of systematic
drift in $\phi^m_t$. This is not the same as the \emph{task-level}
event ``committing now would yield a wrong answer.'' The task-level
false-positive rate additionally depends on (i) the looseness of the
bound for the trajectory distribution at hand and (ii) the correlation
between $s_t$ and answer correctness. Section~\ref{sec:exp-sc} reports
task-level rates separately: $\alpha_m$ is monotone in the task-level
FP rate but not numerically equal to it, as is standard for any
anytime-valid sequential test.

\section{Additional Experiment Results}
\label{appx:exp}
\subsection{Full Experiment Results}
Figure~\ref{fig:fig5-per-cell} shows the token-use--accuracy Pareto envelopes broken down by all 12 model--dataset combinations (3 models $\times$ 4 datasets), for both the upper-only and dual stopping settings. The aggregate advantage reported in Figure~\ref{fig:fig5} holds consistently at the per-cell level: the proposed method with the stability signal dominates CT-Conf and CT-EAT across the majority of the token budget range in every combination, despite using zero calibration trajectories where CT uses 50 per split.

\begin{figure*}[!h]
\centering
\includegraphics[width=0.48\textwidth]{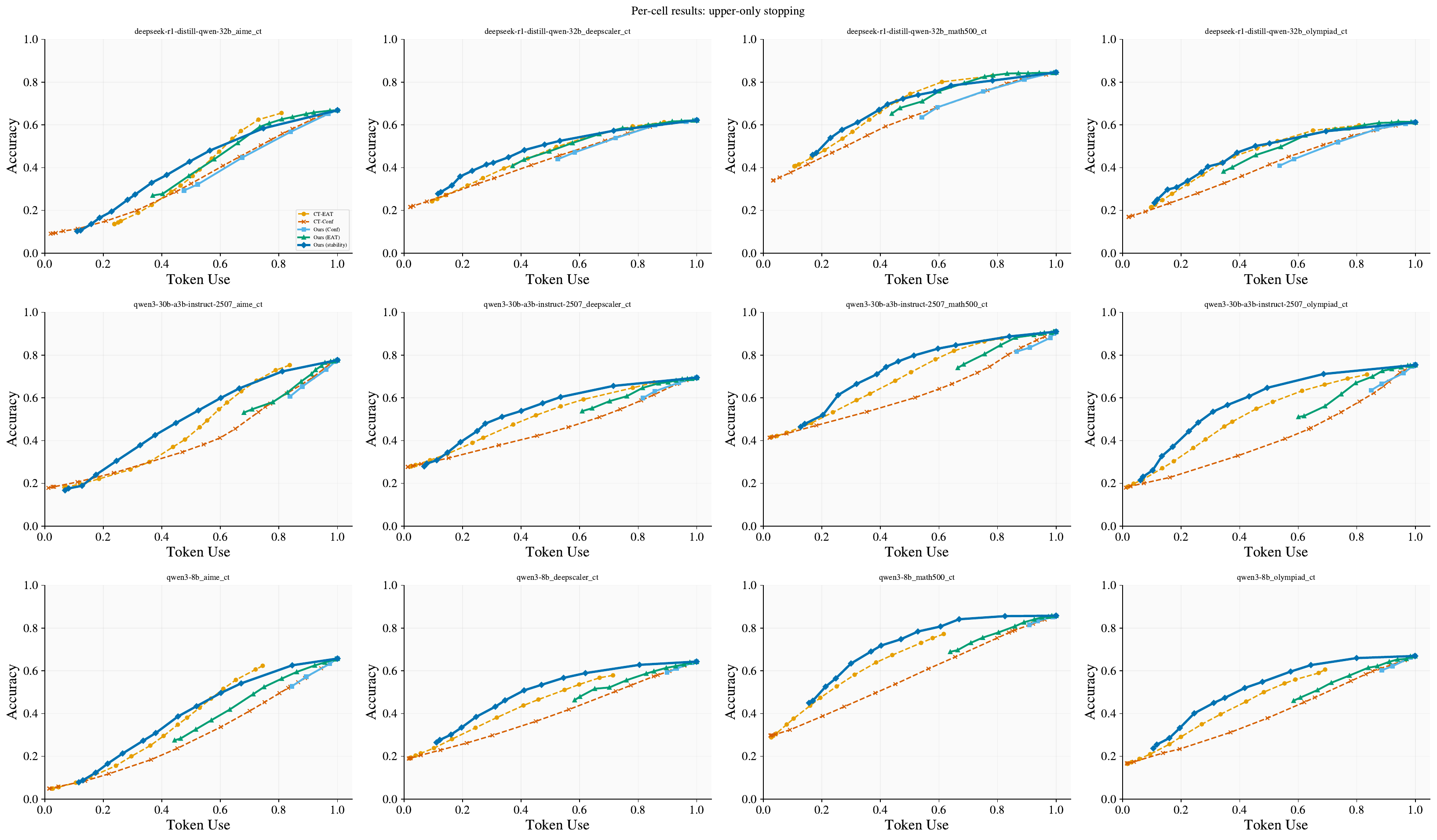}
\hfill
\includegraphics[width=0.48\textwidth]{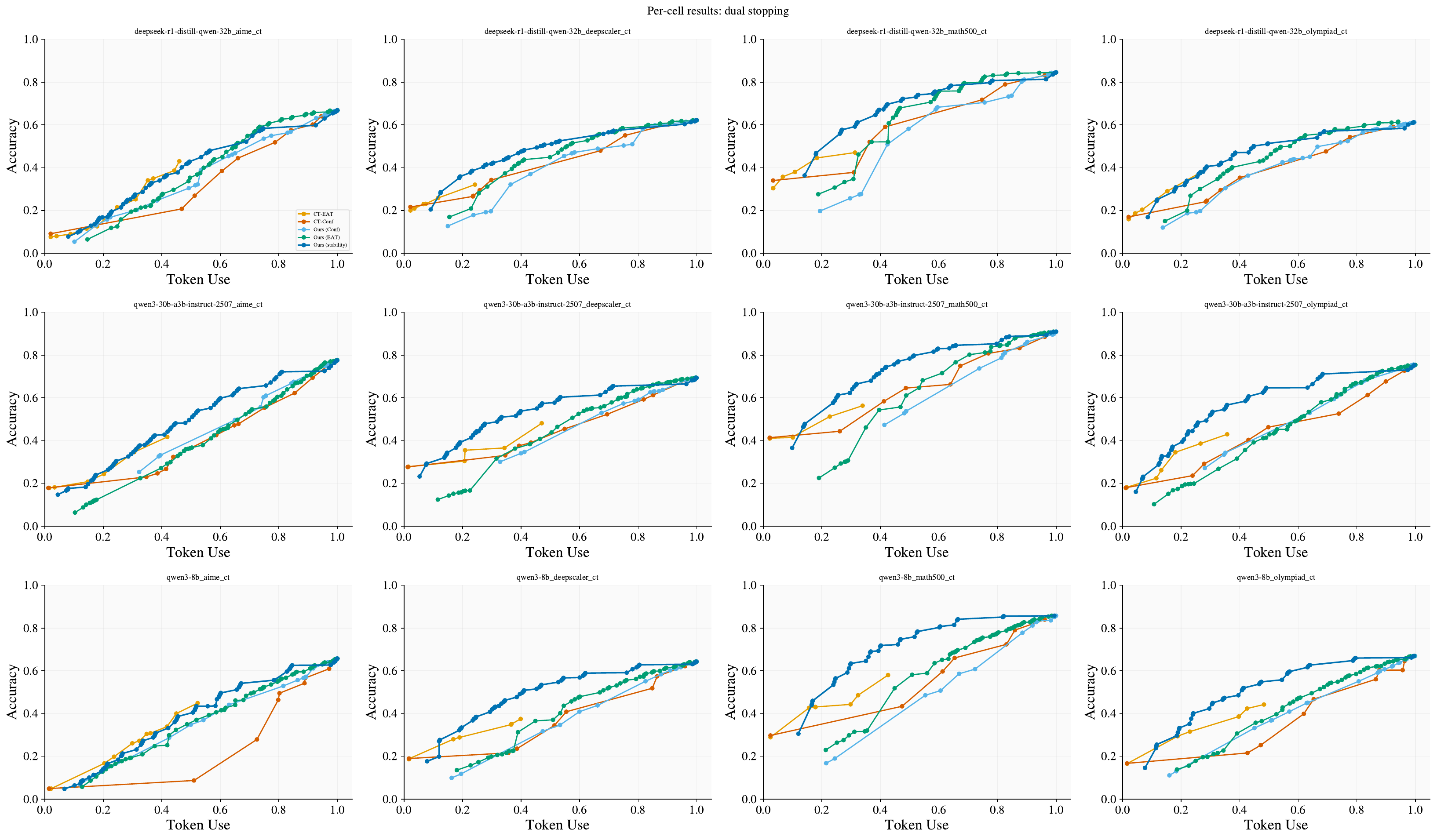}
\caption{Per-combination Pareto envelopes (upper-only, left; dual, right). Best viewed zoomed in.}
\label{fig:fig5-per-cell}
\end{figure*}

Figure~\ref{fig:cross-cell-pareto} shows the aggregated token-use--accuracy Pareto curve under cross-cell deployment in the dual stopping setting, averaged over all $12 \times 11 = 132$ source$\to$target pairs. CT-EAT transferred from a source cell to a different target cell degrades substantially relative to the same-cell reference: the calibration bound breaks when the deployment distribution differs from the calibration distribution. The proposed method requires no calibration data at any stage and uniformly dominates the cross-cell CT baseline across the full operating range.

\begin{figure}[h]
\centering
\includegraphics[width=0.55\textwidth]{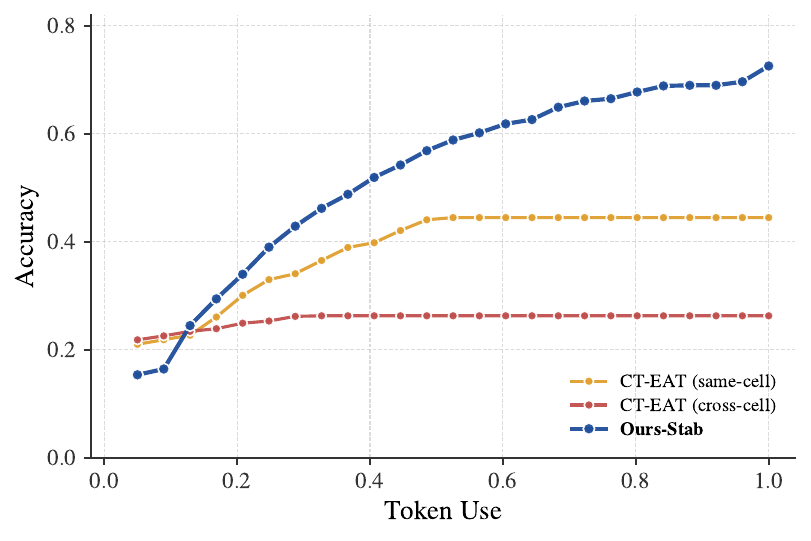}
\caption{Token-use vs.\ accuracy Pareto curves under cross-cell deployment (dual stopping), averaged over all $12 \times 11 = 132$ source$\to$target pairs. CT-EAT same-cell uses each target cell's own calibration data as a reference; CT-EAT cross-cell transfers the source-calibrated threshold to a different target cell. The proposed method (Ours-Stab) requires no calibration data and operates identically across all target cells.}
\label{fig:cross-cell-pareto}
\end{figure}

\subsection{Science Reasoning: GPQA-Diamond}

We evaluate GPQA-Diamond \cite{rein2024gpqa}, a graduate-level multiple-choice science benchmark
containing 198 problems, on all three models from the main paper and two
additional models:
DeepSeek-R1-Distill-Llama-8B and
DeepSeek-R1-Distill-Llama-70B \cite{guo2025deepseek}.

Table~\ref{tab:gpqa} reports the dual-stopping accuracy envelope under
normalized token-use constraints, aggregated over all five models.
A dash indicates that no operating point exists at or below the corresponding
budget.

\begin{table}[!h]
\centering
\small
\begin{tabular}{lcccc}
\toprule
Method
& $\leq 0.2$
& $\leq 0.4$
& $\leq 0.6$
& $\leq 0.8$ \\
\midrule
Truncation
& 0.385 & 0.447 & 0.503 & 0.568 \\
CT-Conf
& 0.399 & 0.439 & 0.490 & 0.554 \\
CT-EAT
& 0.407 & 0.456 & 0.489 & 0.506 \\
CT-Stab
& 0.423 & 0.465 & 0.473 & 0.473 \\
CES-Conf
& 0.112 & 0.209 & 0.375 & 0.511 \\
CES-EAT
& -- & 0.367 & 0.538 & \textbf{0.616} \\
CES-Stab
& \textbf{0.424}
& \textbf{0.528}
& \textbf{0.568}
& 0.594 \\
\bottomrule
\end{tabular}
\caption{
Dual-stopping accuracy on GPQA-Diamond, aggregated over five models.
}
\label{tab:gpqa}
\end{table}

CES achieves the best accuracy at every reported budget.
CES-Stab is strongest through token use $\leq 0.6$, whereas CES-EAT becomes
the best policy at $\leq 0.8$.

\subsection{Code Reasoning}

We further evaluate CES on CRUXEval-O \cite{gu2024cruxeval}, a code-output prediction benchmark
containing 800 problems.
Each problem asks the model to reason about the value returned by a Python
function.
Because the target is a short Python literal, the intermediate-answer
extraction and stability machinery used in the main paper applies without
modification.
We evaluate Qwen3-8B and DeepSeek-R1-Distill-Llama-8B and report the
aggregate over the two model--dataset cells.

\begin{table}[!h]
\centering
\small
\begin{tabular}{lcccc}
\toprule
Method
& $\leq 0.2$
& $\leq 0.4$
& $\leq 0.6$
& $\leq 0.8$ \\
\midrule
Truncation
& 0.292 & 0.368 & 0.450 & 0.486 \\
CT-Conf
& 0.286 & 0.348 & 0.425 & 0.469 \\
CT-EAT
& 0.284 & 0.314 & 0.317 & 0.317 \\
CT-Stab
& \textbf{0.339} & 0.339 & 0.339 & 0.339 \\
CES-Conf
& -- & 0.084 & 0.283 & 0.381 \\
CES-EAT
& -- & 0.251 & 0.424 & 0.488 \\
CES-Stab
& 0.289
& \textbf{0.447}
& \textbf{0.484}
& \textbf{0.492} \\
\bottomrule
\end{tabular}
\caption{
Dual-stopping accuracy on CRUXEval-O, aggregated over two models.
}
\end{table}

CT-Stab retains an advantage at the extreme $\leq 0.2$ budget,
whereas CES-Stab is the best policy from the practical token-use regime above
approximately 0.3.
Together with the other benchmarks, our evaluation therefore covers
mathematical reasoning, science multiple choice, and code-output reasoning.

\subsection{Other Backbones Results}

We next evaluate two DeepSeek-R1-Distill \cite{guo2025deepseek} models using Llama-3.1 backbones.
Each model is evaluated on the four mathematical benchmarks from the main
paper together with GPQA-Diamond.

\paragraph{DeepSeek-R1-Distill-Llama-8B.}

\begin{table}[!h]
\centering
\small
\begin{tabular}{lcccc}
\toprule
Method
& $\leq 0.2$
& $\leq 0.4$
& $\leq 0.6$
& $\leq 0.8$ \\
\midrule
Truncation
& 0.141 & 0.261 & 0.390 & 0.485 \\
CT-Conf
& 0.203 & 0.298 & 0.400 & 0.480 \\
CT-EAT
& 0.266 & 0.334 & 0.347 & 0.347 \\
CT-Stab
& 0.261 & 0.340 & 0.364 & 0.364 \\
CES-Conf
& 0.142 & 0.262 & 0.390 & 0.466 \\
CES-EAT
& 0.173 & 0.346 & 0.474 & \textbf{0.538} \\
CES-Stab
& \textbf{0.276}
& \textbf{0.416}
& \textbf{0.485}
& 0.530 \\
\bottomrule
\end{tabular}
\caption{
Dual-stopping accuracy for DeepSeek-R1-Distill-Llama-8B,
aggregated over five datasets.
Full-run accuracy across cells ranges from 0.475 to 0.744.
}
\end{table}

CES-Stab is strongest through token use $\leq 0.6$,
while CES-EAT is best at $\leq 0.8$.

\paragraph{DeepSeek-R1-Distill-Llama-70B.}

\begin{table}[!h]
\centering
\small
\begin{tabular}{lcccc}
\toprule
Method
& $\leq 0.2$
& $\leq 0.4$
& $\leq 0.6$
& $\leq 0.8$ \\
\midrule
Truncation
& 0.236 & 0.350 & 0.503 & 0.620 \\
CT-Conf
& 0.297 & 0.403 & 0.510 & 0.597 \\
CT-EAT
& 0.355 & 0.438 & 0.447 & 0.452 \\
CT-Stab
& 0.363 & 0.437 & 0.442 & 0.442 \\
CES-Conf
& -- & 0.274 & 0.478 & 0.598 \\
CES-EAT
& 0.206 & 0.450 & 0.571 & 0.651 \\
CES-Stab
& \textbf{0.370}
& \textbf{0.526}
& \textbf{0.609}
& \textbf{0.654} \\
\bottomrule
\end{tabular}
\caption{
Dual-stopping accuracy for DeepSeek-R1-Distill-Llama-70B,
aggregated over five datasets.
Full-run accuracy across cells ranges from 0.604 to 0.828;
AIME full-run accuracy is 0.700.
}
\end{table}

CES-Stab achieves the highest accuracy at every reported budget on the 70B
model.
Together with the main-paper models, these results extend the evaluation to
five models spanning 8B--70B parameters and two backbone families.

\section{Computational cost analyses}
\subsection{Position-Aware FLOPs Accounting}
\label{sec:supp_flops}

Token count can overstate the marginal cost of intermediate probes because the
prefix KV cache is reused.
We therefore additionally report an analytic position-aware FLOPs estimate.
For a token processed at context length $L$, we use
\begin{equation}
f(L)
=
2N_{\mathrm{active}}
+
4n_{\mathrm{layers}}
n_{\mathrm{heads}}
d_{\mathrm{head}}L,
\end{equation}
where the first term accounts for weight-matrix multiplications and the second
for KV-cache attention.
For mixture-of-experts models, $N_{\mathrm{active}}$ denotes the number of
active parameters.
Since the reasoning prefix is already cached, the marginal probe FLOPs include
only the forcing suffix and the at-most-24 greedily decoded answer tokens.

Across all cells, the normalized FLOPs overhead of probing is 0.43--0.83
relative to the original trace cost, compared with a normalized token overhead
of 0.45--0.85.
Token accounting is therefore slightly conservative relative to the analytic
FLOPs accounting.

\subsection{FLOPs-Normalized Pareto Results}

\begin{table}[!h]
\centering
\small
\begin{tabular}{lccc}
\toprule
Method
& $\leq 0.3$
& $\leq 0.5$
& $\leq 0.7$ \\
\midrule
Truncation
& 0.277 & 0.435 & 0.548 \\
CT-Conf (probe-free)
& 0.311 & 0.434 & 0.547 \\
CT-EAT
& 0.323 & 0.390 & 0.430 \\
CT-Stab
& 0.343 & 0.390 & 0.408 \\
CES-Stab
& \textbf{0.353}
& \textbf{0.478}
& \textbf{0.566} \\
\bottomrule
\end{tabular}
\caption{
Dual-stopping accuracy at normalized FLOPs budgets,
aggregated over the 12 main-paper model--dataset cells.
}
\label{tab:flops}
\end{table}

CES-Stab remains the strongest policy at every reported FLOPs budget in
Table~\ref{tab:flops}.
The corresponding probe-inclusive token-axis results differ by approximately
0.01--0.02 accuracy over this range.
Beyond approximately 0.8 of the full-trace budget, the setting is no longer a
meaningful early-stopping regime: probe-free methods eventually overtake
probe-based methods because probing overhead dominates when little generation
is removed.

\section{Accuracy-Relevant Early-Commit Errors}

For evaluating accuracy loss due to early stopping, we define an
\emph{accuracy-relevant commit false positive} as a trajectory satisfying all
three conditions:
\begin{enumerate}
    \item the commit stream crosses before the terminal boundary;
    \item the intermediate prediction at the crossing is incorrect; and
    \item the full-budget terminal prediction would have been correct.
\end{enumerate}
Never-correct trajectories therefore do not contribute to this error count.
Using all trajectories in each cell as the denominator, the accuracy-relevant
FP rate is
\begin{align}
0.26 \pm 0.08
&\quad \text{at } \alpha_+=0.1, \\
0.13 \pm 0.05
&\quad \text{at } \alpha_+=0.01.
\end{align}
When the denominator is restricted to final-correct trajectories, the
corresponding rates are
\begin{align}
0.363 \pm 0.127
&\quad \text{at } \alpha_+=0.1, \\
0.192 \pm 0.078
&\quad \text{at } \alpha_+=0.01.
\end{align}
The two normalizations are consistent because final-correct trajectories
constitute 72.6\% of all trajectories on average.
Table~\ref{tab:per_cell_errors} reports the complete per-cell mapping under
both denominator conventions together with the empirical $c_+$ and $c_-$
constants.

\begin{table*}[!h]
\centering
\scriptsize
\setlength{\tabcolsep}{3.5pt}
\begin{tabular}{llccccccc}
\toprule
Model
& Dataset
& Final Acc.
& \multicolumn{2}{c}{$\alpha_+=0.1$}
& \multicolumn{2}{c}{$\alpha_+=0.01$}
& $c_+$
& $c_-$ \\
\cmidrule(lr){4-5}
\cmidrule(lr){6-7}
&
&
&
FP / All
& FP / Final-Correct
& FP / All
& FP / Final-Correct
&
& \\
\midrule
Qwen3-8B
& AIME
& 0.659 & 0.354 & 0.538 & 0.165 & 0.251 & 246 & 6.90 \\
Qwen3-8B
& MATH-500
& 0.858 & 0.168 & 0.196 & 0.068 & 0.079 & 451 & 0.80 \\
Qwen3-8B
& OlympiadBench
& 0.669 & 0.230 & 0.344 & 0.096 & 0.144 & 569 & 0.76 \\
Qwen3-8B
& DeepScaler
& 0.643 & 0.212 & 0.329 & 0.105 & 0.163 & 622 & 2.03 \\
\midrule
Qwen3-30B-A3B
& AIME
& 0.777 & 0.361 & 0.464 & 0.188 & 0.242 & 419 & 119.90 \\
Qwen3-30B-A3B
& MATH-500
& 0.910 & 0.170 & 0.187 & 0.082 & 0.090 & 158 & 0.80 \\
Qwen3-30B-A3B
& OlympiadBench
& 0.754 & 0.309 & 0.409 & 0.171 & 0.226 & 840 & 0.79 \\
Qwen3-30B-A3B
& DeepScaler
& 0.692 & 0.242 & 0.350 & 0.135 & 0.196 & 494 & 105.58 \\
\midrule
R1-Distill-Qwen-32B
& AIME
& 0.671 & 0.408 & 0.608 & 0.250 & 0.373 & 673 & 0.77 \\
R1-Distill-Qwen-32B
& MATH-500
& 0.846 & 0.160 & 0.189 & 0.094 & 0.111 & 481 & 0.82 \\
R1-Distill-Qwen-32B
& OlympiadBench
& 0.613 & 0.234 & 0.383 & 0.126 & 0.206 & 1118 & 0.80 \\
R1-Distill-Qwen-32B
& DeepScaler
& 0.622 & 0.227 & 0.365 & 0.135 & 0.218 & 912 & 0.82 \\
\midrule
\multicolumn{2}{l}{Mean $\pm$ std}
& 0.726
& $0.256 \pm 0.079$ & $0.363 \pm 0.127$
& $0.135 \pm 0.050$ & $0.192 \pm 0.078$
& & \\
\bottomrule
\end{tabular}
\caption{
Per-cell accuracy-relevant commit errors under two denominator conventions,
together with empirical relaxed-null constants $c_+$ and $c_-$.
``FP / All'' divides the accuracy-relevant early-commit count by all
trajectories in the cell, whereas ``FP / Final-Correct'' conditions the
denominator on trajectories whose full-budget terminal answer is correct.
}
\label{tab:per_cell_errors}
\end{table*}

\section{Statement on the Use of Large Language Model}
During the preparation of this work, the authors used a large language model for language editing assistance, including grammar, spelling, and phrasing. All content was reviewed and edited by the authors, who take full responsibility for the manuscript.

\section{License of Existing Assets}
The models used in this work are released under permissive open-source licenses:
Qwen3-8B and Qwen3-30B-A3B-Instruct-2507 are released under the Apache 2.0 license,
and DeepSeek-R1-Distill-Qwen-32B is released under the MIT license. The benchmark
datasets are used under their respective distribution terms: MATH-500 is derived
from the MATH benchmark and the PRM800K MATH split, whose public repositories are
MIT-licensed; OlympiadBench is distributed on Hugging Face under Apache 2.0
(with the official evaluation repository released under MIT); and DeepScaleR is
released under the MIT license. For AIME, we use publicly available historical
competition problems for academic evaluation only. All models and benchmark
assets are used solely for research evaluation and in a manner consistent with
their stated licenses and intended academic use.

\end{document}